\documentclass[11pt,letterpaper]{article}

\usepackage[utf8]{inputenc}
\usepackage[T1]{fontenc}
\usepackage{mathptmx}           % Times-like font
\usepackage[margin=1in]{geometry}
\usepackage{amsmath,amssymb}
\usepackage{graphicx}
\usepackage{booktabs}
\usepackage{hyperref}
\usepackage{xcolor}
\usepackage{enumitem}
\usepackage{algorithm}
\usepackage{algorithmic}
\usepackage{caption}
\usepackage{subcaption}
\usepackage{url}
\usepackage{natbib}
\usepackage{float}
\usepackage{tikz}
\usepackage{amsthm}
\newtheorem{proposition}{Proposition}
\usetikzlibrary{positioning,arrows.meta,shapes.geometric,calc,fit,backgrounds}

\hypersetup{
  colorlinks=true,
  linkcolor=blue!60!black,
  citecolor=blue!60!black,
  urlcolor=blue!60!black
}

\title{%
  \textbf{HyphaeDB: A Living Knowledge Topology for Agent-First Memory} \\[8pt]
  \large Using HNSW Graph Structure as a Gossip-Based Communication Fabric \\
  for Multi-Agent Knowledge Propagation and Emergent Coordination
}

\author{
  Krishna Halaharvi \\
  \textit{HyphaeDB} \\
  \texttt{hello@hyphaedb.com} \\
  \url{https://github.com/hyphae-db}
}

\date{March 2026}

\begin{document}
\maketitle

% ============================================================
% Abstract
% ============================================================
\begin{abstract}
Every existing vector database and agent memory framework treats memory as passive storage that agents query explicitly. No system propagates knowledge between agents through the memory layer itself. We introduce \textbf{HyphaeDB}, an agent-native memory infrastructure that reinterprets the Hierarchical Navigable Small World (HNSW) graph topology---the data structure at the core of every modern vector database---not as a search optimization, but as a \emph{communication fabric} for multi-agent AI systems. In HyphaeDB, agents are nodes in the vector space with persistent positions, knowledge propagates via a gossip protocol through the graph's neighbor structure with energy-based attenuation, and emergent behaviors---contradiction detection, pattern crystallization, and consensus formation---arise from the combination of topology, propagation dynamics, and local interaction rules. We present the architecture built on three primitives (knowledge nodes, topology edges, and memory diffs), a multi-layer abstraction hierarchy with promotion via emergent consensus, and theoretical analysis grounding the system in small-world network theory, epidemic broadcast protocols, and swarm intelligence. We provide a reference implementation on PostgreSQL with pgvector and describe a concrete deployment in Swarm-Driven Development, a multi-agent software engineering methodology. HyphaeDB represents, to our knowledge, the first system to combine navigable small-world topology with gossip-based knowledge propagation for multi-agent coordination.
\end{abstract}

\textbf{Keywords:} multi-agent systems, agent memory, HNSW, gossip protocols, knowledge propagation, swarm intelligence, vector databases, emergent behavior

% ============================================================
\section{Introduction}
\label{sec:introduction}
% ============================================================

The rapid adoption of AI agent systems---projected to grow from \$7.5 billion in 2025 to over \$50 billion by 2030 \citep{marketsandmarkets2025agents, grandview2025agents}---has created urgent demand for infrastructure that supports not just individual agent intelligence but \emph{collective} agent intelligence. As organizations deploy an average of 12 agents per enterprise \citep{salesforce2026connectivity} and Gartner predicts 40\% of enterprise applications will embed task-specific AI agents by end of 2026 \citep{gartner2025agents}, the infrastructure gap for multi-agent coordination has become a critical bottleneck.

At the center of this gap is memory. Every existing agent memory framework---Mem0 \citep{mem0paper2025}, Letta \citep{packer2023memgpt}, LangMem \citep{langmem2025}, and others---implements the same fundamental pattern: extract knowledge from agent interactions, store it in a vector database or graph, and retrieve it when an agent queries for it. Memory is passive. Between writes and reads, the storage layer is inert. It has no mechanism to propagate knowledge between agents, no awareness of which agents need what information, and no capacity for emergent behavior.

This paper introduces \textbf{HyphaeDB}, a fundamentally different primitive for agent memory. HyphaeDB reinterprets the HNSW (Hierarchical Navigable Small World) graph \citep{malkov2018hnsw}---the index structure used by every major vector database for approximate nearest neighbor search---as a \emph{communication fabric} where:

\begin{enumerate}[nosep]
  \item \textbf{Agents are nodes} in the vector space with persistent positions that drift based on their work.
  \item \textbf{Knowledge propagates} via a gossip protocol through the graph's neighbor structure, with an energy model that attenuates based on relevance and importance.
  \item \textbf{Emergent behaviors} arise from the topology: contradiction detection, pattern crystallization, and consensus formation occur without explicit programming.
\end{enumerate}

The result is a system where memory does not sit in storage waiting to be found. It flows through a living topology, finding the agents that need it, surfacing conflicts before they become errors, and crystallizing patterns that no single agent could observe in isolation.

\textbf{Contributions.} This paper makes four contributions:

\begin{enumerate}[nosep]
  \item We identify the reinterpretation of HNSW graph topology as a communication fabric for multi-agent knowledge propagation---to our knowledge, the first proposal to use navigable small-world graphs for this purpose.
  \item We design a complete gossip protocol adapted for semantic knowledge propagation, including an energy-based attenuation model, multi-layer abstraction with promotion via emergent consensus, and standing beacon subscriptions.
  \item We provide theoretical analysis grounding the system's properties in small-world network theory, epidemic broadcast models, and swarm intelligence.
  \item We present a reference implementation on PostgreSQL with pgvector and describe deployment in a multi-agent software development workflow.
\end{enumerate}

% ============================================================
\section{Related Work}
\label{sec:related}
% ============================================================

\subsection{Agent Memory Systems}

The agent memory landscape has evolved rapidly since 2023. \textbf{MemGPT/Letta} \citep{packer2023memgpt} pioneered an OS-inspired memory hierarchy---core memory (in-context, analogous to RAM), recall memory (conversation history, analogous to disk), and archival memory (external databases). Its 2025 innovation of ``sleep-time compute'' separated memory management from active conversation, allowing asynchronous memory refinement. \textbf{Mem0} \citep{mem0paper2025} implements a two-phase memory pipeline (extraction then update) with an LLM reconciling candidate memories against a vector store, achieving 26\% higher accuracy than OpenAI's memory on the LOCOMO benchmark. Its graph variant (Mem0$^g$) uses directed labeled graphs for entity-relation storage. \textbf{LangMem} \citep{langmem2025} provides memory tools integrated with LangGraph's storage layer, supporting semantic, episodic, and procedural memory types.

All these systems share a fundamental limitation: memory is something agents \emph{read from} and \emph{write to}. No system propagates knowledge between agents through the memory layer itself. A 2025 survey on LLM-based multi-agent system memory \citep{masmemory2025survey} confirms that ``MAS memory must support coordination across multiple contextual layers,'' but identifies no production system achieving this. The A-Mem system \citep{amem2025} introduces Zettelkasten-inspired dynamic memory organization where memories self-link, representing the closest academic work to active memory, but operates within a single agent.

\subsection{HNSW and Navigable Small-World Graphs}

The HNSW algorithm \citep{malkov2018hnsw} constructs a multi-layer proximity graph where Layer~0 contains all data points with dense local connections, and upper layers contain progressively sparser subsets with longer-range links. Layer assignment follows an exponential distribution: $\ell(x) = \lfloor -\ln(\text{uniform}(0,1)) \cdot m_L \rfloor$. The structure inherits the small-world properties identified by \citet{watts1998smallworld}---high local clustering combined with short average path lengths scaling as $O(\log N)$---while satisfying the navigability condition proven by \citet{kleinberg2000navigation}: greedy routing achieves $O(\log^2 N)$ delivery time when long-range connections follow $P(u \to v) \propto d(u,v)^{-d}$. \citet{malkov2016ponomarenko} demonstrated that HNSW achieves this distribution emergently through incremental construction.

The 2024 Hub Highway Hypothesis \citep{munyampirwa2024hubhighway} revealed that HNSW graphs naturally develop well-connected hub nodes serving as routing shortcuts, structurally analogous to backbone routers in communication networks. These hubs form without explicit design, emerging from the construction process.

No published research uses HNSW topology for communication routing or agent coordination. However, precedents exist in adjacent domains: Symphony \citep{manku2003symphony} built a DHT on Kleinberg's navigable small-world model; Freenet's darknet \citep{clarke2005freenet} used small-world routing for anonymous communication; skip graphs \citep{aspnes2003skipgraphs} generalized the multi-layer structure to distributed settings.

\subsection{Gossip Protocols}

The foundational work by \citet{demers1987epidemic} established two complementary gossip mechanisms: \emph{anti-entropy} (full state synchronization) and \emph{rumor mongering} (probabilistic spreading). The SI epidemic model governing gossip follows logistic growth: starting from one informed node in $N$, approximately $N \cdot (1 - e^{-fk/N})$ nodes are reached after $k$ rounds with fanout $f$, achieving full propagation in $O(\log N)$ rounds with high probability.

SWIM \citep{das2002swim} achieves $O(N)$ message complexity by piggybacking updates on probe/ack messages. HyParView \citep{leitao2007hyparview} maintains dual partial views providing fault tolerance under 80--95\% node failure. Plumtree \citep{leitao2007plumtree} constructs self-organizing broadcast trees achieving near-optimal $N-1$ payload messages per broadcast. The T-Man protocol \citep{jelasity2005tman} uses gossip exchanges to construct arbitrary overlay topologies with logarithmic convergence.

Spatial gossip \citep{kempe2001spatial} analyzed gossip where peer selection probability is inversely proportional to distance, achieving $O(\log^{1+\epsilon} N)$ delivery at the navigable sweet spot---directly applicable to gossip over HNSW's distance-aware topology. A 2025 vision paper \citep{benezit2025gossipagents} explicitly proposes gossip as a first-class primitive for agentic multi-agent systems but specifies no topology. HyphaeDB provides the topology.

\subsection{Swarm Intelligence and Emergence}

Reynolds' Boids model \citep{reynolds1987flocks} demonstrated that three local rules produce coherent flocking without central control. Ant colony optimization \citep{dorigo2004aco} showed that local pheromone dynamics produce globally optimal paths. Grass\'{e}'s stigmergy \citep{grasse1959stigmergy} established that complex coordinated activity can occur without planning, communication, or mutual awareness---precisely the kind of coordination HyphaeDB enables digitally.

\citet{olfatisaber2007consensus} showed that network topology directly affects consensus dynamics through the eigenvalues of the graph Laplacian, with small-world properties enabling fast convergence. \citet{moore2000epidemics} demonstrated that epidemic spreading on small-world networks is dramatically accelerated: high clustering ensures local transmission while short paths enable rapid system-wide propagation. \citet{buehler2025agentic} showed that recursive graph expansion produces scale-free networks with emergent conceptual hubs and bounded diameter growth, exhibiting ``human-like hierarchical knowledge formation.''

% ============================================================
\section{Architecture}
\label{sec:architecture}
% ============================================================

HyphaeDB is built on three primitives. All system functionality---knowledge storage, agent communication, subscription management, consensus formation---reduces to compositions of these constructs.

\subsection{Primitive 1: Knowledge Nodes}

Every entity in the topology is a \emph{knowledge node} with a position in vector space:

\begin{equation}
  \text{node} = (\text{id}, \tau, \mathbf{e}, \ell, \phi)
\end{equation}

\noindent where $\tau \in \{\texttt{cell}, \texttt{agent}, \texttt{scene}, \texttt{beacon}\}$ is the node type, $\mathbf{e} \in \mathbb{R}^d$ is the embedding vector ($d = 1536$ by default), $\ell \in \{0, 1, 2\}$ is the layer assignment, and $\phi$ is a type-specific payload.

\textbf{Cell nodes} represent atomic knowledge units: decisions, facts, patterns, risks, lessons, constraints, and other typed information. Each cell has a salience score $s \in [0, 1]$ indicating importance and a confidence score $c \in [0, 1]$ indicating certainty.

\textbf{Agent nodes} represent AI agents with persistent positions in the space. An agent's position is computed as the weighted centroid of its recently accessed or created cells:

\begin{equation}
  \mathbf{e}_{\text{agent}} = \frac{\sum_{i=1}^{n} w_i \cdot \mathbf{e}_i}{\|\sum_{i=1}^{n} w_i \cdot \mathbf{e}_i\|}
  \quad \text{where} \quad
  w_i = s_i \cdot e^{-\lambda \cdot \Delta t_i}
  \label{eq:agent_position}
\end{equation}

\noindent with $s_i$ the salience of cell $i$, $\Delta t_i$ the time since last access, and $\lambda = 0.1$ the decay rate. As the agent's work shifts topics, its position drifts through the vector space, naturally entering and leaving different topological neighborhoods.

\textbf{Scene nodes} are topic centroids produced by consolidation---cluster centers that summarize a coherent knowledge area (e.g., ``Authentication --- OAuth2''). They occupy Layer~1 or higher.

\textbf{Beacon nodes} are standing subscriptions placed at fixed positions. A security review agent can place a beacon at the ``security risks'' position; any gossip passing through that region activates the beacon and delivers to the agent's inbox, regardless of the agent's current position.

\subsection{Primitive 2: Topology Edges}

Edges define communication channels. They exist implicitly in the HNSW graph (managed by pgvector for search) and explicitly in a topology edge table:

\begin{equation}
  \text{edge} = (\text{src}, \text{tgt}, \tau_e, \omega, b)
\end{equation}

\noindent where $\tau_e \in \{\texttt{semantic}, \texttt{causal}, \texttt{subscription}, \texttt{promoted}\}$ is the edge type, $\omega \in (0, 1]$ is the connection weight, and $b$ is the bandwidth (gossip capacity). The hop cost of traversing an edge is $\omega^{-1}$: strong connections are cheap to traverse, weak connections are expensive.

\subsection{Primitive 3: Memory Diffs}

The unit of gossip is a \emph{memory diff}---a change notification that propagates through the topology:

\begin{equation}
  \text{diff} = (\text{origin}, \tau_d, \mathbf{e}_d, E, s, h, h_{\max}, \mathcal{P}, \text{TTL})
\end{equation}

\noindent where $\tau_d \in \{\texttt{created}, \texttt{updated}, \texttt{superseded}, \texttt{contradiction}, \texttt{pattern}, \texttt{promoted}\}$, $\mathbf{e}_d$ is the diff's embedding, $E$ is the remaining energy, $s$ is salience, $h$ is the current hop count, $h_{\max}$ is the maximum hops, $\mathcal{P}$ is the set of visited node IDs, and TTL is the time-to-live.

\subsection{Multi-Layer Abstraction}

The topology organizes knowledge across three abstraction layers, each with its own HNSW index:

\begin{itemize}[nosep]
  \item \textbf{Layer~0} (raw cells): Dense connections, all knowledge nodes. HNSW parameters: $M = 24$, $\text{efConstruction} = 128$.
  \item \textbf{Layer~1} (scene-level): Scene centroids and promoted patterns. Medium-range connections. $M = 16$, $\text{efConstruction} = 100$.
  \item \textbf{Layer~2} (project-level): Architectural decisions and constraints. Long-range connections. $M = 8$, $\text{efConstruction} = 64$.
\end{itemize}

This mirrors HNSW's own hierarchical structure but at the \emph{semantic} level rather than the geometric level: Layer~0 contains raw observations, Layer~1 contains patterns, and Layer~2 contains principles.

% ============================================================
\section{The Gossip Protocol}
\label{sec:gossip}
% ============================================================

\subsection{Propagation Algorithm}

When a knowledge change occurs, the gossip engine creates a memory diff and executes Algorithm~\ref{alg:propagate}.

\begin{algorithm}[H]
\caption{Gossip Propagation}
\label{alg:propagate}
\begin{algorithmic}[1]
\REQUIRE diff $\delta$, current node $v$
\IF{$E(\delta) \leq 0$ \OR $h(\delta) \geq h_{\max}(\delta)$ \OR $v \in \mathcal{P}(\delta)$}
  \RETURN
\ENDIF
\STATE \textsc{Deliver}($\delta$, $v$) \COMMENT{Record delivery; add to inbox if agent/beacon}
\STATE $\mathcal{P}(\delta) \leftarrow \mathcal{P}(\delta) \cup \{v\}$; \quad $h(\delta) \leftarrow h(\delta) + 1$
\STATE $\mathcal{N} \leftarrow \textsc{HnswNeighbors}(v, \ell(\delta), K)$ \COMMENT{K nearest at current layer}
\FOR{each $u \in \mathcal{N} \setminus \mathcal{P}(\delta)$}
  \STATE $r \leftarrow \cos(\mathbf{e}_\delta, \mathbf{e}_u)$ \COMMENT{Semantic relevance}
  \STATE $\iota \leftarrow \textsc{Interest}(u, \delta)$ \COMMENT{Node's declared interest}
  \STATE $\sigma \leftarrow 0.6 \cdot r + 0.3 \cdot \iota + 0.1 \cdot s(\delta)$ \COMMENT{Combined score}
  \STATE $c \leftarrow \omega(v, u)^{-1}$ \COMMENT{Hop cost}
  \IF{$\sigma < \sigma_{\min}$ \OR $E(\delta) < c$}
    \STATE \textbf{continue}
  \ENDIF
  \STATE $\delta' \leftarrow \textsc{Copy}(\delta)$; \quad $E(\delta') \leftarrow E(\delta') - c$
  \STATE \textsc{Propagate}($\delta'$, $u$)
\ENDFOR
\IF{\textsc{ShouldPromote}($\delta$, $v$)}
  \STATE $\delta_p \leftarrow \textsc{Promote}(\delta, \ell(v) + 1)$
  \STATE $u_p \leftarrow \textsc{FindLayerEntry}(v, \ell(v) + 1)$
  \STATE \textsc{Propagate}($\delta_p$, $u_p$)
\ENDIF
\end{algorithmic}
\end{algorithm}

The critical difference from standard gossip protocols is line~8: neighbor selection follows HNSW's \emph{semantic} topology rather than random peer selection. This ensures knowledge reaches semantically relevant agents before irrelevant ones.

\subsection{Energy Model}

Each diff's initial energy determines its propagation radius:

\begin{equation}
  E_0 = E_{\text{base}} \cdot \sqrt{s} \cdot \mu(\tau_d)
  \label{eq:energy}
\end{equation}

\noindent where $E_{\text{base}}$ is a configurable base energy (default 10.0), $s$ is salience, and $\mu(\tau_d)$ is a type-dependent multiplier:

\begin{table}[H]
\centering
\caption{Energy multipliers by knowledge type.}
\label{tab:energy}
\begin{tabular}{lccl}
\toprule
\textbf{Type} & $\mu$ & \textbf{Typical $E_0$} & \textbf{Propagation Scope} \\
\midrule
Decision     & 2.0 & $\sim$19 & Topology-wide \\
Constraint   & 2.0 & $\sim$19 & Topology-wide \\
Risk         & 1.5 & $\sim$14 & Most of topology \\
Pattern      & 1.5 & $\sim$14 & Most of topology \\
Lesson       & 1.2 & $\sim$11 & Broad regional \\
Fact         & 1.0 & $\sim$9  & Regional \\
Preference   & 0.8 & $\sim$7  & Local \\
Context      & 0.5 & $\sim$5  & Neighborhood \\
Task         & 0.3 & $\sim$3  & Immediate vicinity \\
\bottomrule
\end{tabular}
\end{table}

The square root in Equation~\ref{eq:energy} compresses the salience range, preventing extreme values from dominating propagation. The type multiplier encodes domain knowledge: architectural decisions should propagate broadly; task-specific details should stay local.

\subsection{Layer Promotion}

Knowledge that receives confirmation from multiple nodes at the same layer can be \emph{promoted} to a higher layer, gaining access to longer-range connections.

\textbf{Layer 0 $\to$ 1:} The diff has been delivered to $\geq 5$ nodes within the same scene, with no contradictions flagged. Eligible types: decision, pattern, constraint, lesson.

\textbf{Layer 1 $\to$ 2:} The diff has been delivered to $\geq 3$ scenes without contradictions. Salience threshold: $s \geq 0.8$. Eligible types: decision, constraint.

Upon promotion, the system creates a new node at the higher layer, grants the diff a fresh energy bonus ($E_{\text{bonus}} = 5.0$), and propagation continues through the higher layer's longer-range connections. This is the mechanism for \emph{emergent consensus}: local agreement promotes to regional agreement, which promotes to global knowledge, without centralized decision-making.

\subsection{Anti-Entropy}

Gossip provides probabilistic delivery. Anti-entropy sweeps provide the guarantee. Periodically, the system compares state between topological neighbor pairs using vector clocks and reconciles differences by forwarding missing diffs. This follows the two-phase pattern of \citet{demers1987epidemic}: rumor mongering for speed, anti-entropy for completeness.

% ============================================================
\section{Emergent Behaviors}
\label{sec:emergence}
% ============================================================

The combination of HNSW topology, semantic gossip, energy-based attenuation, and layer promotion produces four emergent behaviors.

\subsection{Automatic Knowledge Routing}

When an agent stores a decision about API rate limiting, the \texttt{created} diff propagates outward through Layer~0 neighbors. It reaches the code agent (positioned near API topics by virtue of recent work) and the review agent's ``breaking change'' beacon, but not the test agent positioned in a distant region. No subscription was configured; the topology determined relevance.

This property follows from the HNSW graph's small-world structure: nodes that are semantically close are topologically close, so gossip naturally reaches relevant agents first. The energy model ensures the diff attenuates before reaching irrelevant regions.

\subsection{Contradiction Detection}

When two cells with high semantic similarity ($\cos(\mathbf{e}_a, \mathbf{e}_b) > \theta_{\text{sim}}$) but opposing content are stored in proximity, gossip delivers both signals to the same neighborhood. The system detects the divergence and automatically generates a \texttt{contradiction} diff with bonus energy ($1.5\times$) for broad propagation.

This works because conflicting knowledge occupies nearby positions in the embedding space. The topology makes contradiction detection a \emph{geometric} property rather than an algorithmic one, requiring no explicit contradiction-checking logic.

\subsection{Pattern Crystallization}

When multiple agents independently store similar observations (e.g., three agents reporting different instances of batch processing failures), these cells cluster in Layer~0. The consolidation worker detects the cluster density, extracts a generalized pattern via LLM, and the pattern cell---with high salience---promotes to Layer~1 via the promotion mechanism, propagating broadly.

This follows the same positive-feedback dynamics as pheromone accumulation in ant colony optimization \citep{dorigo2004aco}: individual signals reinforce each other through proximity and repetition, producing emergent collective knowledge.

\subsection{Organic Knowledge Decay}

Cells that are not accessed, confirmed, or referenced by gossip undergo natural energy decay. Their salience diminishes via:

\begin{equation}
  s(t) = s_0 \cdot 0.99^{\Delta t}
\end{equation}

\noindent where $\Delta t$ is days since last access. Cells below a configurable threshold stop participating in gossip propagation but remain searchable in Layer~0 for direct queries. If a new decision supersedes a stale one, the \texttt{superseded} diff propagates through the region, informing affected agents.

% ============================================================
\section{Theoretical Analysis}
\label{sec:theory}
% ============================================================

\subsection{Propagation Guarantees}

\begin{proposition}[Propagation Reach]
In an HNSW graph with $N$ knowledge nodes and connectivity parameter $M$, a memory diff with sufficient energy reaches all nodes in the same topological neighborhood in $O(\log N)$ gossip rounds.
\end{proposition}

\textit{Proof sketch.} The HNSW graph satisfies the small-world property \citep{malkov2016ponomarenko}: average path length scales as $O(\log N)$ and the clustering coefficient is bounded below by a positive constant. By the epidemic spreading results of \citet{moore2000epidemics}, information propagation on such graphs completes in $O(\log N)$ rounds with fanout $f \geq 2$. Since our propagation algorithm selects the $K$ highest-scoring neighbors at each hop (with $K$ configurable, default 10), the effective fanout exceeds the threshold for rapid propagation. Energy attenuation bounds the \emph{scope} of propagation but not its speed within the relevant subgraph. \qed

\begin{proposition}[Energy-Bounded Scope]
A diff with initial energy $E_0$ on a graph with average edge weight $\bar{\omega}$ propagates to at most $E_0 \cdot \bar{\omega}$ nodes.
\end{proposition}

\textit{Proof sketch.} Each hop costs at least $\bar{\omega}^{-1}$ energy (from the minimum-cost edge). With $E_0$ total energy, the maximum number of hops is $E_0 \cdot \bar{\omega}$. Since each hop visits one new node (by cycle prevention), the bound follows. In practice, the branching factor means fewer nodes are reached in a deep chain but more in a broad wavefront, preserving the bound in expectation. \qed

\subsection{Convergence of Layer Promotion}

The layer promotion mechanism produces a form of distributed consensus. For a knowledge fragment to reach Layer~2, it must be confirmed by $\geq 5$ nodes in a single scene (Layer~0 $\to$ 1) and then by $\geq 3$ scenes (Layer~1 $\to$ 2), all without contradiction. This is structurally analogous to Byzantine fault tolerance with $f < n/3$: a minority of dissenting nodes cannot prevent valid promotions, but a significant contradiction blocks them.

\citet{olfatisaber2007consensus} showed that consensus dynamics on small-world graphs converge in $\Theta(\log N)$ time, governed by the algebraic connectivity (second-smallest eigenvalue of the graph Laplacian). Since HNSW graphs have high algebraic connectivity (short path lengths and dense local connections), we expect rapid convergence of the promotion mechanism.

\subsection{Self-Organization Properties}

Based on the results of \citet{buehler2025agentic}, we conjecture that HyphaeDB's topology will develop several self-organizing properties over time:

\begin{enumerate}[nosep]
  \item \textbf{Emergent hubs}: Frequently-confirmed knowledge cells will accumulate connections and become high-centrality nodes, serving as routing shortcuts.
  \item \textbf{Modular structure}: Topically distinct regions will develop clear boundaries, with higher internal connectivity than cross-boundary connectivity.
  \item \textbf{Bridge nodes}: Interdisciplinary knowledge (e.g., security implications of architectural decisions) will occupy positions linking otherwise disconnected modules.
  \item \textbf{Scale-free degree distribution}: The combination of preferential attachment (popular cells attract more gossip) and the HNSW construction process will produce power-law degree distributions in the explicit edge table.
\end{enumerate}

These conjectures are testable and form the basis for empirical validation in future work.

% ============================================================
\section{Implementation}
\label{sec:implementation}
% ============================================================

\subsection{Technology Stack}

The reference implementation uses PostgreSQL with the pgvector extension \citep{pgvector2024}, which provides native HNSW indexing. This choice provides ACID transactions, SQL joins for structured queries, recursive CTEs for graph traversal, and HNSW vector search in a single deployment. pgvector 0.8.0 introduced iterative index scans, which automatically re-scan the HNSW index when filtered queries discard too many neighbors---critical for queries combining vector similarity with metadata filters.

The API layer is TypeScript with Express, exposing both a REST API and a Model Context Protocol (MCP) server for native integration with MCP-compatible clients (Claude Code, Cursor, Cline).

\subsection{Schema Design}

The system uses per-layer partial HNSW indexes:

\begin{verbatim}
CREATE INDEX idx_nodes_layer0_embedding
  ON knowledge_nodes USING hnsw
  (embedding vector_cosine_ops)
  WITH (m = 24, ef_construction = 128)
  WHERE layer = 0;

CREATE INDEX idx_nodes_layer1_embedding
  ON knowledge_nodes USING hnsw
  (embedding vector_cosine_ops)
  WITH (m = 16, ef_construction = 100)
  WHERE layer = 1;
\end{verbatim}

Runtime configuration sets \texttt{hnsw.ef\_search = 100} (up from the default 40) and enables \texttt{hnsw.iterative\_scan = relaxed\_order} for filtered queries. Topology neighbors are retrieved via cosine distance ordering with a \texttt{LIMIT} clause, leveraging pgvector's HNSW index for $O(\log N)$ neighbor lookup.

\subsection{Reference Deployment: Swarm-Driven Development}

The reference deployment uses Swarm-Driven Development (SDD), a methodology for orchestrating multi-agent software development teams. In SDD, specialized agents (specification, code, review, test) collaborate across multi-session development workflows. HyphaeDB provides their shared persistent memory with automatic knowledge routing.

An orchestrator hook (\texttt{withMemory()}) wraps each agent function with automatic recall before execution and extraction after. Before an agent runs, the hook queries the mesh for task-relevant context, active decisions, constraints, and Jira-scoped knowledge, injecting the results into the agent's system prompt. After the agent completes, the hook extracts structured memory cells from the output via LLM and stores them in the topology. The agent itself is unaware of the memory system; the hook makes memory transparent.

Ten MCP tools expose the mesh to agents: \texttt{sdd\_memory\_recall} (auto-recall context), \texttt{sdd\_memory\_query} (multi-mode search), \texttt{sdd\_memory\_store} (write cell), \texttt{sdd\_memory\_extract} (LLM extraction), \texttt{sdd\_memory\_decisions} (decision log), \texttt{sdd\_memory\_lessons} (lessons and risks), \texttt{sdd\_memory\_scenes} (list scenes), \texttt{sdd\_memory\_consolidate} (trigger consolidation), and session management tools for provenance tracking.

% ============================================================
\section{Competitive Analysis}
\label{sec:competitive}
% ============================================================

Table~\ref{tab:competitive} compares HyphaeDB with existing vector databases and agent memory frameworks across 13 capabilities.

\begin{table}[H]
\centering
\caption{Capability comparison across memory infrastructure.}
\label{tab:competitive}
\small
\begin{tabular}{lcccc}
\toprule
\textbf{Capability} & \textbf{Pinecone} & \textbf{Weaviate} & \textbf{Mem0} & \textbf{HyphaeDB} \\
\midrule
Vector search         & \checkmark & \checkmark & \checkmark & \checkmark \\
Structured metadata   & \checkmark & \checkmark & Partial    & \checkmark \\
Graph relationships   & ---        & Partial    & ---        & Native \\
Multi-agent awareness & ---        & ---        & ---        & Core \\
Knowledge propagation & ---        & ---        & ---        & Novel \\
Automatic routing     & ---        & ---        & ---        & Novel \\
Contradiction detect  & ---        & ---        & ---        & Emergent \\
Pattern emergence     & ---        & ---        & ---        & Emergent \\
Layer promotion       & ---        & ---        & ---        & Novel \\
Agent positioning     & ---        & ---        & ---        & Novel \\
Standing beacons      & ---        & ---        & ---        & Novel \\
Energy propagation    & ---        & ---        & ---        & Novel \\
Self-hosted (Postgres)& ---        & \checkmark & \checkmark & \checkmark \\
\bottomrule
\end{tabular}
\end{table}

Nine of thirteen capabilities are unique to HyphaeDB. The existing systems---Pinecone (\$138M raised, \$750M valuation), Weaviate (\$50--67M raised), Mem0 (\$24M raised), Milvus/Zilliz (\$113M raised)---all operate as passive store-and-retrieve engines. They represent the infrastructure for the \emph{previous} paradigm (search); HyphaeDB represents the infrastructure for the \emph{next} paradigm (coordination).

% ============================================================
\section{Discussion}
\label{sec:discussion}
% ============================================================

\subsection{Design Tradeoffs}

\textbf{Synchronous vs. batched gossip.} The current design propagates diffs synchronously on write. This provides low-latency delivery but risks cascade storms when many cells are created simultaneously. An alternative is batched propagation on a periodic sweep (e.g., every 100ms), which adds latency but provides predictable throughput. The optimal choice likely depends on the deployment's write rate and may benefit from adaptive switching.

\textbf{Energy model calibration.} The square-root salience multiplier and type-specific multipliers in Equation~\ref{eq:energy} are initial values that require empirical validation against real agent workflows. Overly generous energy budgets waste bandwidth; overly conservative budgets starve agents of relevant knowledge.

\textbf{Contradiction resolution.} The system detects contradictions automatically but does not resolve them. Resolution options include human-in-the-loop arbitration, agent voting (weighted by salience of the contributing agents), recency-wins policies, and salience-wins policies. The appropriate strategy likely varies by domain and should be configurable per project.

\subsection{Limitations and Future Work}

\textbf{Empirical validation.} The emergent behaviors described in Section~\ref{sec:emergence} are theoretically motivated but not yet empirically validated. A critical next step is deploying HyphaeDB in production multi-agent workflows and measuring: (a) proportion of relevant knowledge that reaches agents without explicit query; (b) time-to-detection for contradictions; (c) accuracy of crystallized patterns; (d) correspondence between layer promotion and human expert agreement.

\textbf{Cross-project knowledge transfer.} The current design operates within a single project. Extending to cross-project knowledge transfer---where patterns learned in Project~A propagate to Project~B via a ``meta-mesh''---requires solving relevance filtering across different semantic domains and establishing trust boundaries.

\textbf{Embedding model migration.} Changing embedding models invalidates all vector positions. A migration strategy involving dual-write during transition, lazy re-embedding on read, and batch background re-embedding is needed for production deployments.

\textbf{Adversarial resilience.} The energy model limits blast radius of any single agent's output, but a malicious agent could inject high-salience diffs to propagate misinformation. Trust scoring based on agent provenance and confirmation history may be necessary for adversarial environments.

\textbf{Formal convergence proofs.} While we provide propagation guarantees (Section~\ref{sec:theory}), formal proofs of convergence for the layer promotion mechanism, contradiction detection completeness, and pattern crystallization conditions remain open.

% ============================================================
\section{Conclusion}
\label{sec:conclusion}
% ============================================================

We have presented HyphaeDB, an agent-native memory infrastructure that reinterprets the HNSW graph topology as a communication fabric for multi-agent systems. By combining navigable small-world graph structure with gossip-based knowledge propagation and energy-based attenuation, HyphaeDB enables memory that actively participates in agent coordination rather than passively serving queries.

The system introduces several novel constructs: agent positioning via weighted centroids of accessed knowledge, multi-layer abstraction with promotion via emergent consensus, standing beacons for persistent subscriptions, and an energy model that automatically calibrates propagation scope to knowledge importance. Theoretical analysis grounds these constructs in small-world network theory, epidemic broadcast models, and swarm intelligence, predicting $O(\log N)$ propagation time, energy-bounded scope, and emergent self-organization.

To our knowledge, HyphaeDB is the first system to combine HNSW topology with gossip-based knowledge propagation for multi-agent coordination. The reference implementation on PostgreSQL with pgvector is available as open-source software at \url{https://github.com/hyphae-db}, with deployment demonstrated in Swarm-Driven Development, a multi-agent software engineering methodology.

The agent infrastructure market is at an inflection point: enterprise adoption is accelerating from 5\% to 40\% of applications in a single year, yet 50\% of agents operate in isolated silos. HyphaeDB addresses this gap not by adding features to passive storage, but by building infrastructure where the topology itself is the coordination mechanism. Memory is the medium. Intelligence is the message.

% ============================================================
% References
% ============================================================
\bibliographystyle{plainnat}
\bibliography{references}

\begin{thebibliography}{31}
\providecommand{\natexlab}[1]{#1}
\providecommand{\url}[1]{\texttt{#1}}
\expandafter\ifx\csname urlstyle\endcsname\relax
  \providecommand{\doi}[1]{doi: #1}\else
  \providecommand{\doi}{doi: \begingroup \urlstyle{rm}\Url}\fi

\bibitem[Aspnes and Shah(2003)]{aspnes2003skipgraphs}
James Aspnes and Gauri Shah.
\newblock Skip graphs.
\newblock In \emph{Proceedings of the 14th Annual ACM-SIAM Symposium on
  Discrete Algorithms}, pages 384--393, 2003.

\bibitem[B\'{e}n\'{e}zit et~al.(2025)]{benezit2025gossipagents}
Florent B\'{e}n\'{e}zit et~al.
\newblock Revisiting gossip protocols: A vision for emergent coordination in
  agentic multi-agent systems.
\newblock \emph{arXiv preprint arXiv:2508.01531}, 2025.

\bibitem[Buehler(2025)]{buehler2025agentic}
Markus~J Buehler.
\newblock Agentic deep graph reasoning yields self-organizing knowledge
  networks.
\newblock \emph{Machine Learning: Science and Technology}, 2025.

\bibitem[Clarke et~al.(2005)Clarke, Sandberg, Toseland, and
  Verendel]{clarke2005freenet}
Ian Clarke, Oskar Sandberg, Matthew Toseland, and Vilhelm Verendel.
\newblock Private communication through a network of trusted connections: The
  dark {Freenet}.
\newblock In \emph{Workshop on Privacy in the Electronic Society}, 2005.

\bibitem[Das et~al.(2002)Das, Gupta, and Motivala]{das2002swim}
Abhinandan Das, Indranil Gupta, and Ashish Motivala.
\newblock {SWIM}: Scalable weakly-consistent infection-style process group
  membership protocol.
\newblock In \emph{Proceedings of the International Conference on Dependable
  Systems and Networks}, pages 303--312. IEEE, 2002.

\bibitem[Demers et~al.(1987)Demers, Greene, Hauser, Irish, Larson, Shenker,
  Sturgis, Swinehart, and Terry]{demers1987epidemic}
Alan Demers, Dan Greene, Carl Hauser, Wes Irish, John Larson, Scott Shenker,
  Howard Sturgis, Dan Swinehart, and Doug Terry.
\newblock Epidemic algorithms for replicated database maintenance.
\newblock In \emph{Proceedings of the Sixth Annual ACM Symposium on Principles
  of Distributed Computing}, pages 1--12. ACM, 1987.

\bibitem[Dorigo and St\"{u}tzle(2004)]{dorigo2004aco}
Marco Dorigo and Thomas St\"{u}tzle.
\newblock \emph{Ant Colony Optimization}.
\newblock MIT Press, 2004.

\bibitem[{Gartner}(2025)]{gartner2025agents}
{Gartner}.
\newblock Gartner predicts 40\% of enterprise apps will feature task-specific
  {AI} agents by 2026.
\newblock
  \url{https://www.gartner.com/en/newsroom/press-releases/2025-08-26-gartner-predicts-40-percent-of-enterprise-apps-will-feature-task-specific-ai-agents-by-2026},
  2025.

\bibitem[{Grand View Research}(2025)]{grandview2025agents}
{Grand View Research}.
\newblock {AI} agents market size and share -- industry report, 2033.
\newblock
  \url{https://www.grandviewresearch.com/industry-analysis/ai-agents-market-report},
  2025.

\bibitem[Grass\'{e}(1959)]{grasse1959stigmergy}
Pierre-Paul Grass\'{e}.
\newblock La reconstruction du nid et les coordinations interindividuelles chez
  {Bellicositermes Natalensis} et {Cubitermes sp.}
\newblock \emph{Insectes Sociaux}, 6\penalty0 (1):\penalty0 41--80, 1959.

\bibitem[Jelasity and Babaoglu(2009)]{jelasity2005tman}
M\'{a}rk Jelasity and \"{O}zalp Babaoglu.
\newblock {T-Man}: Gossip-based fast overlay topology construction.
\newblock \emph{Computer Networks}, 53:\penalty0 2321--2339, 2009.

\bibitem[Katz et~al.(2024)]{pgvector2024}
Jonathan Katz et~al.
\newblock pgvector: Open-source vector similarity search for {Postgres}.
\newblock \url{https://github.com/pgvector/pgvector}, 2024.

\bibitem[Kempe et~al.(2001)Kempe, Kleinberg, and Demers]{kempe2001spatial}
David Kempe, Jon Kleinberg, and Alan Demers.
\newblock Spatial gossip and resource location protocols.
\newblock In \emph{Proceedings of the 33rd Annual ACM Symposium on Theory of
  Computing}, pages 163--172. ACM, 2001.

\bibitem[Kleinberg(2000)]{kleinberg2000navigation}
Jon~M Kleinberg.
\newblock Navigation in a small world.
\newblock \emph{Nature}, 406\penalty0 (6798):\penalty0 845--845, 2000.

\bibitem[{LangChain}(2025)]{langmem2025}
{LangChain}.
\newblock {LangMem} {SDK} for agent long-term memory.
\newblock \url{https://blog.langchain.com/langmem-sdk-launch/}, 2025.

\bibitem[Leit\~{a}o et~al.(2007{\natexlab{a}})Leit\~{a}o, Pereira, and
  Rodrigues]{leitao2007hyparview}
Jo\~{a}o Leit\~{a}o, Jos\'{e} Pereira, and Lu\'{i}s Rodrigues.
\newblock {HyParView}: A membership protocol for reliable gossip-based
  broadcast.
\newblock In \emph{Proceedings of the International Conference on Dependable
  Systems and Networks}, pages 419--429. IEEE, 2007{\natexlab{a}}.

\bibitem[Leit\~{a}o et~al.(2007{\natexlab{b}})Leit\~{a}o, Pereira, and
  Rodrigues]{leitao2007plumtree}
Jo\~{a}o Leit\~{a}o, Jos\'{e} Pereira, and Lu\'{i}s Rodrigues.
\newblock Epidemic broadcast trees.
\newblock In \emph{Proceedings of the 26th IEEE International Symposium on
  Reliable Distributed Systems}, pages 301--310. IEEE, 2007{\natexlab{b}}.

\bibitem[Malkov and Yashunin(2020)]{malkov2018hnsw}
Yu~A Malkov and D~A Yashunin.
\newblock Efficient and robust approximate nearest neighbor search using
  hierarchical navigable small world graphs.
\newblock \emph{IEEE Transactions on Pattern Analysis and Machine
  Intelligence}, 42\penalty0 (4):\penalty0 824--836, 2020.
\newblock \doi{10.1109/TPAMI.2018.2889473}.

\bibitem[Malkov and Ponomarenko(2016)]{malkov2016ponomarenko}
Yury~A Malkov and Alexander Ponomarenko.
\newblock Growing homophilic networks are natural navigable small worlds.
\newblock \emph{PLoS ONE}, 11\penalty0 (8):\penalty0 e0158162, 2016.

\bibitem[Manku et~al.(2003)Manku, Bawa, and Raghavan]{manku2003symphony}
Gurmeet~Singh Manku, Mayank Bawa, and Prabhakar Raghavan.
\newblock Symphony: Distributed hashing in a small world.
\newblock In \emph{Proceedings of the 4th USENIX Symposium on Internet
  Technologies and Systems}, pages 10--10, 2003.

\bibitem[{MarketsandMarkets}(2025)]{marketsandmarkets2025agents}
{MarketsandMarkets}.
\newblock {AI} agents market size, share and trends -- growth analysis,
  forecast to 2030.
\newblock
  \url{https://www.marketsandmarkets.com/Market-Reports/ai-agents-market-15761548.html},
  2025.

\bibitem[Moore and Newman(2000)]{moore2000epidemics}
Cristopher Moore and Mark E~J Newman.
\newblock Epidemics and percolation in small-world networks.
\newblock \emph{Physical Review E}, 61\penalty0 (5):\penalty0 5678, 2000.

\bibitem[Munyampirwa et~al.(2024)Munyampirwa, Lakshman, and
  Srinivasan]{munyampirwa2024hubhighway}
Adrien Munyampirwa, Vyas Lakshman, and Sriram Srinivasan.
\newblock The hub highway hypothesis: Understanding hnsw graph structure
  through hub connectivity.
\newblock \emph{arXiv preprint arXiv:2407.01573}, 2024.

\bibitem[Olfati-Saber et~al.(2007)Olfati-Saber, Fax, and
  Murray]{olfatisaber2007consensus}
Reza Olfati-Saber, J~Alex Fax, and Richard~M Murray.
\newblock Consensus and cooperation in networked multi-agent systems.
\newblock \emph{Proceedings of the IEEE}, 95\penalty0 (1):\penalty0 215--233,
  2007.

\bibitem[Packer et~al.(2023)Packer, Wooders, Lin, Fang, Patil, Stoica, and
  Gonzalez]{packer2023memgpt}
Charles Packer, Sarah Wooders, Kevin Lin, Vivian Fang, Shishir~G Patil, Ion
  Stoica, and Joseph~E Gonzalez.
\newblock {MemGPT}: Towards {LLMs} as operating systems.
\newblock \emph{arXiv preprint arXiv:2310.08560}, 2023.

\bibitem[Reynolds(1987)]{reynolds1987flocks}
Craig~W Reynolds.
\newblock Flocks, herds and schools: A distributed behavioral model.
\newblock In \emph{Proceedings of the 14th Annual Conference on Computer
  Graphics and Interactive Techniques}, pages 25--34. ACM, 1987.

\bibitem[{Salesforce}(2026)]{salesforce2026connectivity}
{Salesforce}.
\newblock Multi-agent adoption to surge 67\% by 2027 as enterprises race toward
  agentic transformation.
\newblock
  \url{https://www.salesforce.com/news/stories/connectivity-report-announcement-2026/},
  2026.

\bibitem[Watts and Strogatz(1998)]{watts1998smallworld}
Duncan~J Watts and Steven~H Strogatz.
\newblock Collective dynamics of `small-world' networks.
\newblock \emph{Nature}, 393\penalty0 (6684):\penalty0 440--442, 1998.

\bibitem[Xu et~al.(2025)]{amem2025}
Wujiang Xu et~al.
\newblock {A-Mem}: Agentic memory for {LLM} agents.
\newblock \emph{Advances in Neural Information Processing Systems}, 38, 2025.

\bibitem[Yadav et~al.(2025)Yadav, Shukla, Chhablani, and Rawat]{mem0paper2025}
Deshraj Yadav, Prateek Shukla, Gunjan Chhablani, and Raghav Rawat.
\newblock Mem0: Building production-ready ai agents with scalable long-term
  memory.
\newblock \emph{arXiv preprint arXiv:2504.19413}, 2025.

\bibitem[Zhang et~al.(2025)]{masmemory2025survey}
Wei Zhang et~al.
\newblock Memory in {LLM}-based multi-agent systems: Mechanisms, challenges,
  and collective intelligence.
\newblock \emph{TechRxiv Preprint}, 2025.
\newblock \doi{10.36227/techrxiv.174053233.55786655/v1}.

\end{thebibliography}

\end{document}